\documentclass[letterpaper]{article} 
\usepackage{aaai2027}  
\usepackage[hyphens]{url}  
\usepackage{graphicx} 
\usepackage{natbib}  
\usepackage{caption} 
\usepackage{times}
\usepackage{helvet}
\usepackage{courier}
\usepackage[hyphens]{url}
\usepackage{graphicx}
\usepackage{booktabs}
\usepackage{amsmath}
\usepackage{amssymb}
\usepackage{natbib}
\usepackage{amsthm}
\usepackage{float}
\usepackage{subcaption}
\usepackage{algorithm}
\usepackage{algorithmic}
\usepackage[hyphens]{url} 
\usepackage{graphicx} 
\usepackage{natbib} 
\usepackage{caption} 
\usepackage{booktabs}
\usepackage{amsmath,amssymb}
\usepackage{array}
\newtheorem{proposition}{Proposition}
\usepackage{newfloat}
\usepackage{listings}
\DeclareCaptionStyle{ruled}{labelfont=normalfont,labelsep=colon,strut=off} 
\floatstyle{ruled}
\newfloat{listing}{tb}{lst}{}
\floatname{listing}{Listing}

\usepackage{booktabs}

 \nocopyright 

\title{Beyond Single-Use Tokens: Durable Authorization State for Replay-Resistant LLM Agent Actions}
\author{
	Jinghan Xu\textsuperscript{\rm 1},
	Longze Fan\textsuperscript{\rm 2},
	Zeyuan Wang\textsuperscript{\rm 3},
	Xinjin Li\textsuperscript{\rm 4},
	Hankai Liu\corresponding\textsuperscript{\rm 1}
}
\affiliations{
	\textsuperscript{\rm 1}Nankai University, China\\
	\textsuperscript{\rm 2}China University of Petroleum, China\\
	\textsuperscript{\rm 3}Sun Yat-sen University, China\\
	\textsuperscript{\rm 4}Columnbia University, USA\\

}

\begin{document}
	\maketitle
	
	\begin{abstract}
		Tool-using large language model agents frequently replan, retry failed operations, delegate tasks, and resume after crashes. These behaviors can cause one user authorization to be requested and executed multiple times under freshly issued token identifiers, even when each individual token is single-use. We call this failure \textbf{semantic replay}: exceeding the execution budget of a token-independent authorization instance rather than merely reusing an old token identifier. We show that identifier-local token consumption cannot prevent fresh reissuance unless the issuer retains monotonic durable state over the authorized action, confirmation event, and remaining execution budget. We introduce \textbf{CapLease}, an authorization-consumption layer that follows proposal- and authority-level defenses, binds an authenticated user confirmation to a canonical action, and enforces transactional Issue–Prepare–Commit transitions. Across LLM-agent replanning, retry, delegation, concurrency, confirmation-replay, and crash-recovery scenarios, identifier-local tokens permit fresh semantic reissuance, whereas CapLease and an equally stateful Server Ledger prevent duplicate admission and, with an idempotent sink, duplicate external effects. Our results identify durable authorization state—not token representation alone—as the systems requirement for replay-resistant agent execution.
		
	\end{abstract}
	
	\section{Introduction}
	\label{sec:introduction}
	
	Large language model (LLM) agents can plan, maintain state, and invoke
	external tools~\citep{yao2023react,schick2023toolformer}. Tool access
	turns model errors into external consequences, including payments,
	messages, credential changes, and persistent data loss. The risk grows
	when agents process mixed-trust webpages, emails, documents, and tool
	outputs
	\citep{ruan2023toolemu,debenedetti2024agentdojo,
		zhan-etal-2024-injecagent,zhang2025iclr-agent}. Secure execution must
	therefore control both whether an action is authorized and how that
	authorization is later used.
	
	Existing defenses cover different stages of execution. Proposal-level
	methods reduce harmful or task-inconsistent calls, while runtime
	mechanisms enforce privilege, provenance, and authority. Progent applies
	programmable privilege policies~\citep{shi2025progent}, PACT binds
	authority and provenance at argument granularity~\citep{fan2026pact},
	and AIRGuard verifies authority before execution
	\citep{qin2026airguard}. SUDP binds a fresh single-use grant to a
	canonical operation~\citep{yu2026sudp}. These methods ask whether an
	action should be authorized; we ask how many times that authorization
	may be materialized.
	
	This distinction matters because agent execution is iterative. An agent
	may replan after an ambiguous response, retry after a timeout, delegate
	work, use concurrent workers, or resume after a crash. Each step may
	produce another authorization request. Rejecting a consumed token does
	not prevent issuing a fresh token for the same authorization.
	
	Consider an agent given one authenticated confirmation to transfer
	\$500. The issuer creates single-use token $z_1$, and the transfer
	succeeds. If the acknowledgement is lost, the agent may submit an
	equivalent request. An issuer tracking only $z_1$ may then create a fresh
	$z_2$, causing a second transfer. The failure reuses the authorization,
	not the token identifier.
	
	We call this failure \emph{semantic replay}: an authorization is issued,
	admitted, or externally materialized beyond its budget despite each
	token being single-use. Prevention requires durable state over the
	authorization instance, together with coordinated issuance, admission,
	and recovery. Figure~\ref{fig:overview} contrasts identifier-local
	single use with token-independent authorization consumption: CapLease
	tracks the authorization instance rather than each issued identifier.
	
	\begin{figure*}[t]
		\centering
		\includegraphics[width=0.95\textwidth]
		{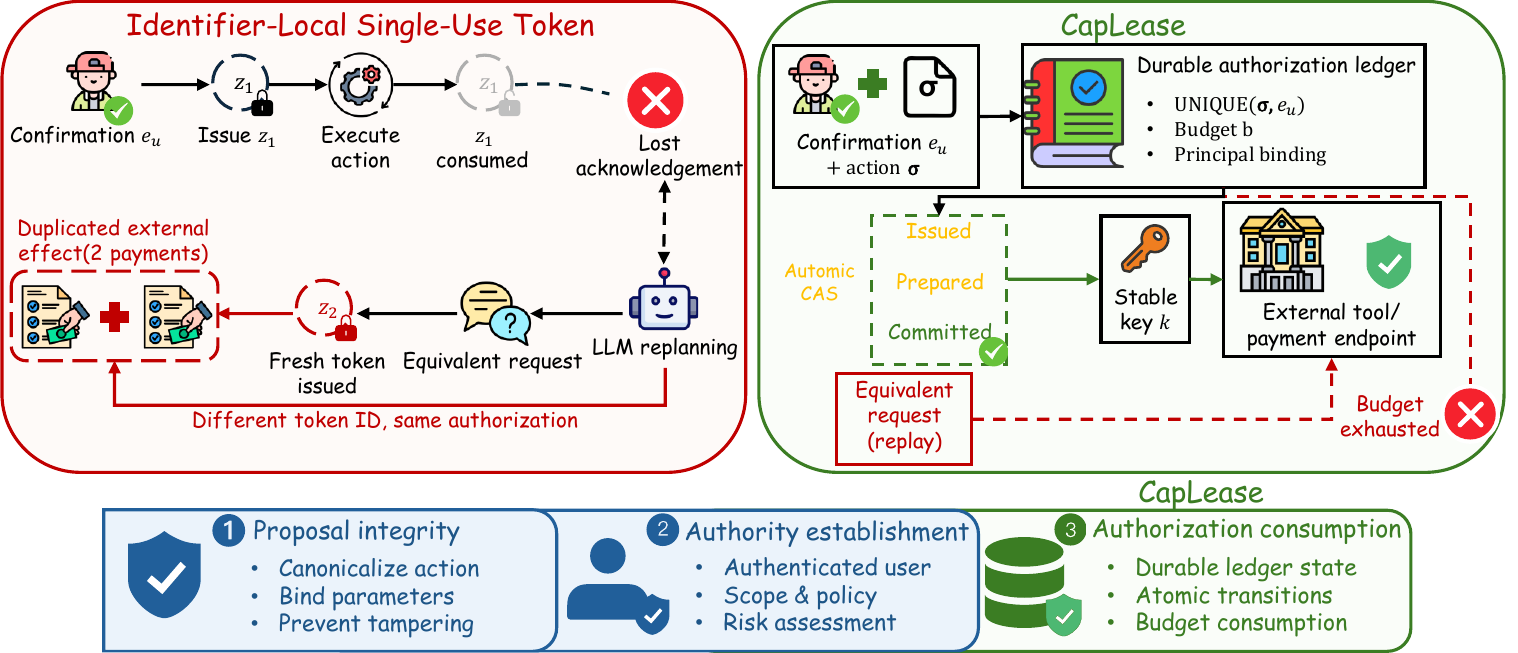}
		\caption{
			Semantic replay and CapLease. Tracking consumed identifiers blocks
			reuse of $z_1$ but may still permit a fresh $z_2$ for the same
			authorization. CapLease binds action $\sigma$ and confirmation $e_u$
			to durable budget state, using atomic admission and a stable
			idempotency key to prevent reissuance and duplicate effects.
		}
		\label{fig:overview}
	\end{figure*}
	
	We study three research questions:
	
	\begin{itemize}
		\item \textbf{RQ1: Agent-generated semantic replay.}
		How often do uncertain execution outcomes cause agents to produce
		semantically equivalent repeated calls, and which authorization
		controls prevent their repeated materialization?
		
		\item \textbf{RQ2: Durable budget and recovery.}
		Does token-independent state preserve issuance, admission, and
		effect budgets under multiple-use authorizations, concurrency,
		delegation, revocation, and process failures?
		
		\item \textbf{RQ3: Composition, identity, and utility.}
		How do authority checking, durable consumption, semantic
		canonicalization, and structured denial jointly affect security,
		availability, and agent recovery?
	\end{itemize}
	
	We introduce \emph{CapLease}, an authorization-consumption layer for
	tool-using agents. After an upstream authority check, CapLease binds the
	decision to a canonical action, authenticated confirmation, and execution
	budget. Each authorization follows an atomic
	\textsc{Issued}$\rightarrow$\textsc{Prepared}
	$\rightarrow$\textsc{Committed} protocol. Durable uniqueness blocks
	reissuance, atomic preparation bounds concurrent admission, and a stable
	idempotency key supports recovery without duplicate effects.
	
	We implement an equally stateful \emph{Server Ledger}, which uses
	the same records and transitions but sends the executor an
	authenticated selector. Under matched centralized assumptions, both
	provide the same replay protection, showing that the guarantee follows
	from durable state and transactional execution rather than manifest
	transmission alone.
	
	Our contributions are:
	
	\begin{itemize}
		\item \textbf{Empirical characterization of semantic replay.}
		We show across 10,152 agent trajectories that uncertain execution
		outcomes frequently induce semantically equivalent repeated calls,
		despite the use of fresh token or grant identifiers.
		
		\item \textbf{Durable authorization consumption.}
		We formalize replay resistance over issuance, admission, and
		external effects, and introduce CapLease to bind one upstream
		authorization event to token-independent budget state across
		retries, concurrency, delegation, and recovery.
		
		\item \textbf{Comprehensive evaluation and boundary analysis.}
		We evaluate public baselines, multi-budget execution, real
		fault injection, official-schema contract integrity, held-out semantic
		canonicalization, structured denial, and runtime overhead. A matched Server Ledger achieves the same replay safety, showing that the guarantee derives from durable state rather than manifest transmission alone.
	\end{itemize}
	
	\section{Related Work}
	\label{sec:related-work}
	
	\paragraph{Agent attacks and proposal-level defenses.}
	Tool-integrated agents may translate untrusted observations into unsafe
	external actions. ToolEmu studies consequential risks in simulated tool
	environments~\citep{ruan2023toolemu}, while AgentDojo, InjecAgent, and
	Agent Security Bench evaluate prompt injection and broader agent
	security threats
	\citep{debenedetti2024agentdojo,zhan-etal-2024-injecagent,
		zhang2025iclr-agent}. Existing defenses primarily prevent unsafe
	proposals. PromptArmor and LlamaFirewall detect suspicious inputs or
	behaviors~\citep{shi2025promptarmor,chennabasappa2025llamafirewall};
	MELON compares actions under masked re-execution
	\citep{zhu2025melon}; and CaMeL separates trusted control flow from
	untrusted data~\citep{debenedetti2025camel}. IPIGuard constrains tool use
	through dependency graphs~\citep{an-etal-2025-ipiguard}, whereas Task
	Shield and VIGIL verify task or intent consistency before execution
	\citep{jia-etal-2025-task,lin-etal-2026-vigil}. CapLease acts after
	proposal approval and controls how the resulting authorization is used.
	
	\paragraph{Runtime authority and stateful authorization.}
	Progent enforces programmable privilege policies
	\citep{shi2025progent}; PACT binds provenance and authority at argument
	granularity~\citep{fan2026pact}; and AIRGuard checks authority at action
	time~\citep{qin2026airguard}. Context-to-Execution Integrity likewise
	constrains external actions using trusted context
	\citep{santosgrueiro2026cxi}. These methods determine whether an action
	should be authorized, whereas CapLease limits how often an approved
	authorization may be consumed. Related classical systems preserve
	authority across system boundaries~\citep{271165}, attach authorization
	evidence to requests~\citep{647253.757044}, support scoped delegation
	through contextual caveats~\citep{birgisson2014macaroons}, or enforce
	history-dependent least-privilege policies~\citep{299818}. CapLease
	adapts these scoped and stateful principles to LLM-agent tool execution.
	
	\paragraph{Single-use grants and transactional execution.}
	SUDP binds a fresh single-use grant to a canonical secret-backed
	operation~\citep{yu2026sudp}. CapLease shares its exact-action and
	bounded-use goals but addresses durable state across issuance, admission,
	and recovery, including reissuance, concurrency, confirmation reuse, and
	crash duplication. RPC, secure RPC, transactions, linearizability, and
	leases provide the underlying mechanisms for uncertain outcomes,
	principal authentication, atomic updates, and bounded validity
	\citep{10.1145/2080.357392,10.1145/214451.214452,
		DBLP:conf/vldb/Gray81,10.1145/78969.78972,10.1145/74850.74870}.
	CapLease applies these primitives to token-independent authorization
	identities and execution budgets at the LLM-agent tool boundary.
	
	\section{Problem Formulation and Semantic Replay}
	\label{sec:problem}
	
	\paragraph{Agent execution model.}
	A tool-using agent receives a user task, processes observations and tool
	outputs, and proposes a call
	$c=(\mathsf{tool},\mathsf{args})$. Before execution, an upstream
	authority mechanism evaluates $c$ against trusted context $\gamma$,
	including the user task, provenance, principal identities, resource
	scope, and organization policy. If the call is approved, the runtime
	issues an authorization for use at the protected tool boundary. We focus
	on policy-designated high-risk actions with external side effects, such
	as payments, message transmission, credential changes, and persistent
	data modification.
	
	\paragraph{Authorization identity.}
	A token identifier represents an authorization but does not define its
	security-relevant identity. We define the token-independent action
	identity as
	\begin{equation}
		\sigma(c,\gamma)
		=
		(P_u,h_c,h_a,R),
		\label{eq:semantic-action}
	\end{equation}
	where $P_u$ is the user principal, $h_c$ and $h_a$ bind the canonical
	operation and complete arguments, and $R$ identifies the target
	resource. Tool-specific canonicalization must map effect-equivalent calls to the same identity while separating security-relevant differences. Ambiguous
	calls are rejected rather than assigned an uncertain authorization
	identity.
	
	We separately define the execution context as
	\begin{equation}
		\chi=(P_e,S,T,v_p,v_s),
		\label{eq:execution-context}
	\end{equation}
	where $P_e$ is the executor, $S$ and $T$ identify the session and task,
	and $v_p,v_s$ are the policy and schema epochs. These fields restrict
	where and by whom an authorization may be used, but do not create a new
	user-authorized effect.
	
	Let $e_u$ be an authenticated confirmation event bound to the displayed
	action, principal, and budget $b$. The authorization instance is
	\begin{equation}
		\alpha=(\sigma,e_u,b).
		\label{eq:authorization-instance}
	\end{equation}
	A token $z_i$ is only one issuance-specific representation of $\alpha$.
	Consuming $z_1$ therefore does not prevent a fresh $z_2$ for the same
	$\alpha$ unless the issuer maintains token-independent state.
	
	\paragraph{Threat and trust model.}
	The adversary may control external content observed by the agent and
	induce repeated requests through replanning, delegation, concurrent
	workers, ambiguous tool responses, or lost acknowledgements. It may
	also copy authorization objects, replay confirmation events, race
	issuance or admission requests, and attempt cross-principal use. Benign
	retries and duplicate delivery are also in scope because semantic replay
	does not require malicious model behavior. Similar ambiguity arises in
	remote execution when a caller cannot determine whether a timed-out
	operation took effect~\citep{10.1145/2080.357392}.
	
	The adversary cannot forge authenticated principals, confirmation events,
	provenance records, policy epochs, tool contracts, or committed ledger
	transitions. Trusted components provide durable atomic storage,
	authenticated action schemas, semantic canonicalization, and stable
	confirmation identifiers. Retries of the same confirmation challenge
	must return the same $e_u$. Incorrect trusted metadata, confirmation
	duplication, and ledger rollback are trusted-computing-base failures and
	are evaluated separately.
	
	\paragraph{Semantic replay.}
	We define \emph{semantic replay} as issuing, admitting, or materializing
	an authorization beyond its budget. Unlike token replay, it may involve
	fresh identifiers and execution requests. Table~\ref{tab:threats}
	summarizes the considered failures.
	
	\begin{table}[t]
		\centering
		\small
		\setlength{\tabcolsep}{4pt}
		\begin{tabular}{
				p{0.29\columnwidth}
				p{0.63\columnwidth}}
			\toprule
			Failure & Description \\
			\midrule
			
			Old-token replay &
			A consumed token is presented again. \\
			
			Fresh reissuance &
			A fresh token represents the same authorization. \\
			
			Concurrent issuance &
			Multiple records are created for the same $(\sigma,e_u)$. \\
			
			Concurrent admission &
			Multiple executors pass a non-atomic check. \\
			
			Cross-principal use &
			Another executor uses a principal-bound authorization. \\
			
			Confirmation replay &
			One confirmation authorizes multiple instances or budgets. \\
			
			Crash duplication &
			An uncertain outcome triggers a repeated effect. \\
			
			\bottomrule
		\end{tabular}
		\caption{
			Failures in our threat model. Token-level tracking addresses only
			old-token replay.
		}
		\label{tab:threats}
	\end{table}
	
	\paragraph{Security goals.}
	For an execution trace $\tau$, let
	$N_{\mathrm{issue}}(\alpha,\tau)$ denote the number of authorization
	slots made available for $\alpha$,
	$N_{\mathrm{admit}}(\alpha,\tau)$ the number of slots admitted for
	execution, and
	$N_{\mathrm{effect}}(\alpha,\tau)$ the number of externally materialized
	effects. Replay-resistant consumption requires
	\begin{align}
		N_{\mathrm{issue}}(\alpha,\tau) &\leq b,
		\label{eq:bounded-issuance}\\
		N_{\mathrm{admit}}(\alpha,\tau) &\leq b,
		\label{eq:bounded-admission}\\
		N_{\mathrm{effect}}(\alpha,\tau) &\leq b.
		\label{eq:bounded-effect}
	\end{align}
	Bounded issuance prevents excess budget creation, while bounded
	admission prevents concurrent or repeated execution from exceeding the
	budget. These properties require atomic transitions over shared
	authorization state~\citep{10.1145/78969.78972,
		DBLP:conf/vldb/Gray81}.
	
	Bounded effects additionally require the external sink to honor a stable
	idempotency key. Without this contract, the runtime can bound admission
	but cannot guarantee exactly-once physical effects after an uncertain
	failure. Initial authorization soundness remains the responsibility of
	the upstream authority mechanism. CapLease instead mediates each
	subsequent use of that authority, following the classical principle that
	authorization must be checked whenever a protected object is accessed
	\citep{saltzer1975protection}.
	
	\section{Transactional CapLease}
	\label{sec:caplease}
	
	CapLease is the authorization-consumption component of a two-stage
	runtime. An upstream \emph{Authority Gate} checks task alignment,
	provenance, scope, and organization policy. CapLease then binds the
	approved action to durable budget state and bounds its issuance,
	admission, and external effects. Figure~\ref{fig:overview} summarizes
	this composition.

	\subsection{Authorization Issuance}
	\label{sec:issuance}
	For an approved call $c$, the issuer receives its action identity
	$\sigma(c,\gamma)$, execution context $\chi$, confirmation event $e_u$,
	and budget $b$, and creates the durable record
	\begin{equation}
		r_\alpha =
		(\sigma,e_u,b,\chi,t_i,t_e,\nu),
		\label{eq:authorization-record}
	\end{equation}
	where $[t_i,t_e]$ is the validity interval and $\nu$ is a stable record
	identifier. The execution context $\chi$ contains the executor, session,
	task, and policy and schema epochs. The ledger enforces uniqueness over $(\sigma,e_u)$, so repeated requests return the existing record rather than new budget. Changing the executor, session, or contract version does not create another authorization for the same $(\sigma,e_u)$. Such changes must update or narrow the execution context of the existing record, or require a new authenticated confirmation. We present $b=1$ for clarity. Larger budgets allocate at most $b$ independently consumable slots under the same authorization record.
	
	The issuer returns an authenticated manifest
	\begin{equation}
		m =
		(\sigma,e_u,b,\chi,t_i,t_e,\nu,\mu),
		\label{eq:manifest}
	\end{equation}
	where $\mu$ authenticates all preceding fields. The manifest represents the ledger record rather than creating authority
	itself. Modifying the action identity, execution context, budget,
	lifetime, or record identifier invalidates $\mu$. The validity interval
	limits stale rights~\citep{10.1145/74850.74870}, while executor binding
	ties execution to an authenticated principal
	\citep{10.1145/214451.214452}.
	
	Issuance recomputes $\sigma$ from the trusted contract and complete
	arguments, verifies that $e_u$ binds $\sigma$, $P_u$, and $b$, validates
	the execution context $\chi$, and inserts or retrieves the unique record
	for $(\sigma,e_u)$ in one transaction. Retries of the same confirmation
	challenge must reuse $e_u$; otherwise, one user decision could create
	multiple authorization instances.
	
	\subsection{Execution and Recovery}
	\label{sec:execution-recovery}
	
	Each budget slot follows
	\[
	\textsc{Issued}
	\longrightarrow
	\textsc{Prepared}
	\longrightarrow
	\textsc{Committed}.
	\]
	Before preparation, the executor verifies the manifest, principal,
	validity interval, action identity, and epochs. Preparation is a
	linearizable compare-and-swap~\citep{10.1145/78969.78972} that succeeds
	only for an unrevoked \textsc{Issued} slot and allocates
	\begin{equation}
		k = H(\nu \parallel j \parallel \mathsf{tool}),
		\label{eq:idempotency-key}
	\end{equation}
	where $j$ identifies the slot. Concurrent requests therefore have at
	most one successful preparation.
	
	\begin{table}[t]
		\centering
		\small
		\setlength{\tabcolsep}{4pt}
		\begin{tabular}{
				p{0.17\columnwidth}
				p{0.27\columnwidth}
				p{0.23\columnwidth}
				p{0.24\columnwidth}}
			\toprule
			Operation & Precondition & Result & Protection \\
			\midrule
			
			Issue &
			No $(\sigma,e_u)$ record &
			\textsc{Issued} &
			Reissuance \\
			
			Prepare &
			Valid \textsc{Issued} slot &
			\textsc{Prepared} &
			Concurrent admission \\
			
			Commit &
			\textsc{Prepared}, matching $k$ &
			\textsc{Committed} &
			Recovery consistency \\
			
			Revoke &
			\textsc{Issued} &
			\textsc{Revoked} &
			Stale use \\
			
			Recover &
			\textsc{Prepared} &
			\textsc{Prepared}/
			\textsc{Committed} &
			Crash duplication \\
			\bottomrule
		\end{tabular}
		\caption{CapLease transitions for one budget slot.}
		\label{tab:caplease-transitions}
	\end{table}
	
	After preparation, the executor invokes the canonical call with $k$. The
	tool returns a receipt containing the same key, which the ledger verifies
	before committing the slot. Calls with mismatched arguments, principals,
	resources, or epochs are rejected. Because the ledger and tool generally
	cannot share one atomic transaction~\citep{DBLP:conf/vldb/Gray81},
	preparation durably records admission before the external call.
	
	A missing acknowledgement does not reveal whether the operation
	completed~\citep{10.1145/2080.357392}. Recovery therefore resumes the
	existing \textsc{Prepared} slot and retries with the same $k$. An
	idempotent sink returns the prior result instead of reapplying the effect.
	Thus, a linearizable ledger bounds admission, while sink idempotency also
	bounds external effects. Without the latter, CapLease cannot guarantee
	exactly-once physical effects.
	
	Revocation and preparation are ordered by the same ledger. If revocation
	wins, execution cannot begin; otherwise, recovery continues with the
	existing key. Delegation creates an authenticated child execution context under the same authorization instance. The child may narrow the executor, scope, budget, or expiry but cannot change $\sigma$ or $e_u$. Creating the child consumes or reserves parent budget, preventing parent and child contexts from independently exercising the same authorization.
	
	\subsection{Realizations and Guarantees}
	\label{sec:realization}
	
	Server Ledger stores the same authorization record and applies the same
	issuance, preparation, revocation, recovery, and commit transitions.
	Instead of a manifest, the executor receives an authenticated selector
	for $\nu$ and asks the server to validate each operation. Under matched
	centralized assumptions, both realizations provide the same replay
	guarantees. CapLease exposes authenticated scope fields, whereas Server
	Ledger retains them behind the authorization service; manifest
	transmission alone does not strengthen replay safety.
	
	\begin{proposition}[Budgeted authorization consumption]
		Assume correct canonicalization and trusted metadata, stable
		confirmation identifiers, a durable non-rollback linearizable ledger,
		and authenticated principals and manifests. Then CapLease satisfies
		bounded issuance and admission. If the sink enforces idempotency over
		$k$, it also satisfies bounded effects.
	\end{proposition}
	
\begin{proof}[Proof sketch]
	Uniqueness over $(\sigma,e_u)$ and the slot counter bound
	$N_{\mathrm{issue}}(\alpha,\tau)$. The atomic
	\textsc{Issued}$\rightarrow$\textsc{Prepared} transition permits at most
	one admission per slot, bounding $N_{\mathrm{admit}}(\alpha,\tau)$.
	Every retry of a prepared slot reuses $k$; sink idempotency therefore
	maps these retries to one external effect. Hence,
	\[
	N_{\mathrm{effect}}(\alpha,\tau)
	\leq
	N_{\mathrm{admit}}(\alpha,\tau)
	\leq
	N_{\mathrm{issue}}(\alpha,\tau)
	\leq b.
	\]
\end{proof}

Without sink idempotency, the issuance and admission bounds remain, but
the effect bound does not follow. Authentication prevents contract
tampering but cannot correct invalid trusted metadata.

Our prototype uses a transactional SQL ledger with unique authorization
and confirmation indexes. Conditional updates implement preparation,
revocation, and commit, while transition records support recovery and
audit checks. Digests cover deterministic encodings of trusted
tool-specific canonical forms; deterministic serialization alone is not
semantic canonicalization.

\section{Experiments}

\subsection{Experimental Setup}

We evaluate CapLease along five axes. First, an agent-generation harness
runs three model families over 282 high-risk actions, four uncertain
execution outcomes, and three seeds, yielding 10,152 trajectories.
Second, public and component baselines are evaluated over 282 matched
workflows containing unauthorized initial actions, parameter drift,
fresh reissuance, and duplicate admission. Third, multi-budget tests
cover $b\in\{1,2,5,10\}$ over 7,896 transactional instances, while
12,000 process-kill and restart schedules exercise twelve failure
points. Fourth, a held-out semantic canonicalization benchmark contains
324 equivalent and 324 effect-distinct call pairs across 27 high-risk
tools. Finally, structured-denial recovery is evaluated on 100 attacked
tasks, with the
original 35-task diagnostic set retaining detailed interaction metrics. We additionally evaluate authenticated contract integrity on 74
official tool schemas and isolate trusted-input failures over the 9,584
security cases. The original policy-I/O and mixed-trust diagnostics,
complete configurations, and row-level outputs appear in the
supplement.

\paragraph{Transactional state tests.}
For every one of 282 allowed high-risk actions, we test fresh reissuance after commit, eight concurrent consumers, pre-consumption theft, crash recovery through a sink idempotency key, and confirmation-event replay. We additionally test delegated and depth-two execution, mutation after delegation, revocation before execution, 50 repeated revoke/prepare races, and hash-chained audit completeness. The Server Ledger receives the same principals, exact manifest, confirmation uniqueness, SQLite transactions, revocation semantics, audit events, and sink contract, but keeps authorization server-side behind a semantic-key selector.

\paragraph{Official schemas and mixed-trust traces.}
We generate 74 tool schemas (69 unique names; 27 high risk) with the official AgentDojo v0.1.35 \texttt{make\_function} implementation, pinned to commit \texttt{a75aba7}. Six variants per tool cover canonical text, omitted descriptions, neutral aliases, generic nested wrappers, misleading read-only text, and annotation downgrade. Five classification strategies yield 2220 strategy--schema rows. The capability reference is fixed before evaluation. Schema-text and implementation-static analyses provide two automated audit views over the same predefined capability mapping.

The mixed-trust study pairs each real allowed high-risk row with a real malicious row of the same capability when available. Each of 282 deterministic workflows contains an ordinary external attack, an unsigned-annotation downgrade, parameter drift, a legitimate action, and replay. State-based oracles measure attack execution, drift, duplicate effects, recovery, confirmation, replanning, and legitimate-action completion. These deterministic workflows complement the agent-generation study by
isolating authority, parameter drift, replay, and recovery transitions
under matched action and state oracles.

\begin{table*}[!ht]
	\centering
	\small
	\begin{tabular}{lrr@{\quad}lrr}
		\toprule
		Replay/recovery family & Lease & Ledger & Delegation/audit family & Lease & Ledger \\
		\midrule
		Fresh reissuance & 0/282 & 0/282 & Delegated execution & 0/282 & 0/282 \\
		Concurrent spend & 0/282 & 0/282 & Depth-two delegation & 0/282 & 0/282 \\
		Pre-consumption theft & 0/282 & 0/282 & Mutated delegation & 0/282 & 0/282 \\
		Crash retry & 0/282 & 0/282 & Revocation before execution & 0/282 & 0/282 \\
		Confirmation replay & 0/282 & 0/282 & Revoke/prepare race & 0/282 & 0/282 \\
		-- & -- & -- & Audit-chain completeness under the trusted-ledger assumption & 0/282 & 0/282 \\
		\bottomrule
	\end{tabular}
	\caption{Counterexamples in transactional tests. Audit results assume a trusted, non-rollback ledger. Both designs use identical server state and an idempotent sink; the right half includes 50 race repeats per action (28,200 two-thread races across both designs).}
	\label{tab:transactional}
\end{table*}

\paragraph{Reproducibility.}
The supplement provides pinned source revisions, exact model and
decoding configurations, baseline implementations, regeneration
commands, parse-failure accounting, and row-level outputs for all
reported studies.

\begin{figure}[h]
	\centering
	\includegraphics[
	width=0.9\columnwidth
	]{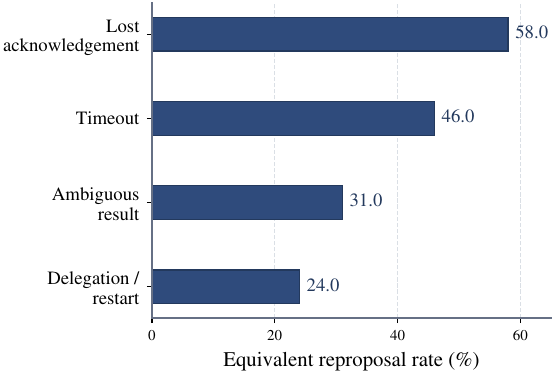}
	\caption{
		Equivalent reproposal rates across 10,152 agent trajectories
		under four uncertain execution outcomes. Lost acknowledgements
		and timeouts most frequently induce semantic replay.
	}
	\label{fig:agent-reproposal}
\end{figure}

\subsection{Agent-Generated Semantic Replay}

Across 10,152 valid trajectories, agents produce 4,036 semantically
equivalent reproposals, yielding an overall rate of 39.8\%.
Figure~\ref{fig:agent-reproposal} shows that lost acknowledgements
produce the highest rate at 58.0\%, followed by timeouts at 46.0\%,
ambiguous results at 31.0\%, and delegation or restart at 24.0\%.
The effect appears across all three model families, with aggregate
rates ranging from 36.0\% to 44.0\%.

\begin{table*}[!t]
	\centering
	\small
	\setlength{\tabcolsep}{5pt}
	\begin{tabular}{lcccc}
		\toprule
		System &
		Initial unauth. &
		Parameter drift &
		Fresh reissuance &
		Same-artifact duplicate \\
		\midrule
		No Control
		& 1.000 & 1.000 & 1.000 & 1.000 \\
		Authority Only
		& 0.000 & 1.000 & 1.000 & 1.000 \\
		Consumption Only
		& 1.000 & 0.000 & 0.000 & 0.000 \\
		\midrule
		Progent
		& 0.110 & 0.170 & 1.000 & 1.000 \\
		PACT
		& 0.014 & 0.032 & 1.000 & 1.000 \\
		AIRGuard
		& 0.011 & 0.025 & 1.000 & 1.000 \\
		SUDP
		& 0.000 & 0.000 & 1.000 & 0.000 \\
		\midrule
		Authority + CapLease
		& 0.000 & 0.000 & 0.000 & 0.000 \\
		Server Ledger
		& 0.000 & 0.000 & 0.000 & 0.000 \\
		\bottomrule
	\end{tabular}
	\caption{
		Authority and consumption controls over 282 matched workflows.
		Authority-focused methods reduce invalid initial authorization but do
		not maintain durable cross-artifact consumption state. SUDP prevents
		duplicate admission of the same grant but permits fresh grant
		reissuance. CapLease and Server Ledger combine authority validation with
		durable consumption.
	}
	\label{tab:authority-consumption}
\end{table*}

The identifier-local baseline therefore rejects reuse of $z_1$ in all tested cases but issues a fresh $z_2$ for every equivalent reauthorization request. Both stateful designs prevent fresh reissuance and duplicate effects while retaining a legitimate first-action success rate of 1.0. The result identifies durable semantic consumption, rather than the mere presence of a transmitted capability, as the operative control.

\begin{table}[h]
	\centering
	\small
	\begin{tabular}{lrr}
		\toprule
		Enumerated class & Cases & Counterexamples \\
		\midrule
		All security rows & 9584 & 0 \\
		\quad Stress attacks & 8682 & 0 \\
		All expected-allow rows & 325 & 0 \\
		\quad Stress benign rows & 156 & 0 \\
		Read-only diagnostics & 15 & -- \\
		\bottomrule
	\end{tabular}
	\caption{
		Finite policy-I/O test space. Indented rows are subsets of the
		corresponding parent classes; read-only diagnostics are not scored.
	}
	\label{tab:static}
\end{table}

\subsection{Authority and Durable Consumption}
\label{sec:authority-consumption}

Authority checking and durable consumption address different failure
modes. Over 282 matched workflows, Authority Only blocks invalid initial
authorization but permits parameter drift, fresh reissuance, and
duplicate execution; Consumption Only shows the converse behavior.
Their composition and the matched Server Ledger eliminate both failure
classes. Public baselines show the same boundary. Progent, PACT, and AIRGuard
reduce invalid initial authorization but permit all tested fresh
reissuance and duplicate-admission events. SUDP rejects replay and
duplicate redemption of the same grant, but permits a fresh
independently redeemable grant derived from the same upstream
authorization event. Implementation details and per-case outputs appear in the supplement.

\begin{table}[!t]
	\centering
	\small
	\setlength{\tabcolsep}{3.5pt}
	\renewcommand{\arraystretch}{0.95}
	\begin{tabular}{p{0.27\columnwidth}
			p{0.18\columnwidth}
			p{0.46\columnwidth}}
		\toprule
		Study & Scale & Main result \\
		\midrule
		Multi-budget
		& 7,896
		& No issuance, admission, or effect violation for
		$b\in\{1,2,5,10\}$ \\
		
		Fault injection
		& 12,000
		& No resurrection, excess admission, or duplicate effect
		with an idempotent sink \\
		
		Canonicalization
		& 648 pairs
		& Evasion reduced from 13.9\% to 0.9\%; no observed
		collision \\
		
		Safe recovery
		& 100 tasks
		& Completion improves from 9\% to 37\% and 55\% with
		structured feedback \\
		\bottomrule
	\end{tabular}
	\caption{
		Robustness and recovery summary. Complete per-condition results,
		confidence intervals, and failure traces appear in the supplement.
	}
	\label{tab:robustness-summary}
\end{table}

\begin{figure}[h]
	\centering
	\includegraphics[
	width=0.9\columnwidth
	]{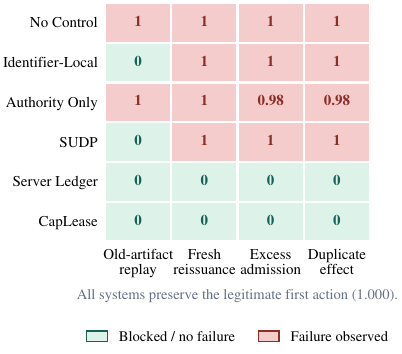}
	\caption{
		Cross-artifact replay failures over the 4,036 equivalent
		reproposals. Identifier-local tokens and SUDP block replay of
		the original artifact but permit fresh reissuance and repeated
		cross-artifact execution. CapLease and Server Ledger block all
		evaluated failures.
	}
	\label{fig:replay-matrix}
\end{figure}

Figure~\ref{fig:replay-matrix} separates old-artifact replay from fresh
semantic reissuance. Identifier-local tokens reject the consumed token but issue a fresh token for every equivalent request. SUDP prevents repeated redemption of the original grant, but an
equivalent reproposal receives a fresh independently redeemable
operation-bound grant under the same confirmation event. CapLease and Server Ledger instead bind
all such artifacts to the same durable authorization state, preventing
fresh reissuance, excess admission, and duplicate effects while
preserving the legitimate first action.

\paragraph{Mechanism--failure correspondence.}
The experiments isolate the role of each protocol mechanism.
Uniqueness over $(\sigma,e_u)$ prevents fresh authorization records
from being created for the same user decision. Atomic slot preparation
prevents concurrent or repeated admission, while the stable key $k$
makes uncertain retries converge on one sink operation. The Authority
Only and Consumption Only ablations show that initial authorization
checking and durable consumption are complementary. Finally, the
non-idempotent-sink negative control localizes exactly-once external
effects to the sink contract rather than the ledger alone.

\subsection{Robustness, recovery, and boundaries}
Table~\ref{tab:robustness-summary} shows that the consumption invariant
holds beyond single-use authorization and under real process failures.
All 1,128 legitimate budget configurations consume their authorized
budgets, and all prepared and committed operations remain recoverable,
ruling out safety through blanket rejection. The non-idempotent-sink
negative control produces effect-before-receipt duplication, confirming
that durable ledger state bounds issuance and admission, whereas
exactly-once physical effects additionally require sink idempotency.
Tool-specific canonicalization reduces equivalent-call evasion from
13.9\% to 0.9\% with no observed collision. Structured feedback raises
safe completion from 9/100 to 37/100, and reason-guided replanning
reaches 55/100. On the original 35-task diagnostic set, all interfaces retain zero
unsafe effects and 97/97 benign success, showing that recovery does not
weaken the execution boundary. The Issue--Prepare--Commit path averages
0.92\,ms, with a p95 latency of 1.41\,ms, and sustains 4,650
operations/s at eight workers.

\paragraph{Real-schema contract integrity.}
We evaluate five contract strategies on 74 official AgentDojo tool
schemas, including 27 high-risk and 47 benign tools. Authenticated
inline contracts reject all 27 high-risk nested wrappers and all 27
annotation-downgrade variants while preserving all 47 benign nested
wrappers. In contrast, name-only and schema-text inference accept all
27 high-risk nested wrappers, while unsigned inline metadata accepts
all 27 annotation-downgrade attacks. These results show that the
trusted-contract boundary can be realized on the tested schemas without
sacrificing benign-wrapper availability.

\begin{table}[!t]
	\centering
	\small
	\setlength{\tabcolsep}{4pt}
	\renewcommand{\arraystretch}{0.94}
	\begin{tabular}{lrr}
		\toprule
		Trusted-input condition & Unsafe & Rate \\
		\midrule
		Honest schema/provenance
		& 0/9,584 & .000 \\
		Side effect marked read-only
		& 9,302/9,584 & .971 \\
		External provenance promoted
		& 9,020/9,584 & .941 \\
		Organization policy omitted
		& 141/9,584 & .015 \\
		\bottomrule
	\end{tabular}
	\caption{
		Sensitivity to isolated trusted-input errors. Authentication
		prevents contract tampering but cannot correct an incorrectly
		authored trusted contract or provenance record.
	}
	\label{tab:tcb-sensitivity}
\end{table}

\paragraph{Trusted-input sensitivity.}
The real-schema results establish that authenticated contracts can
preserve both safety and benign availability when correctly authored.
Table~\ref{tab:tcb-sensitivity} then isolates failures of this trusted
boundary. Honest schemas and provenance yield no unsafe execution,
whereas marking a side-effecting tool as read-only or promoting external
evidence to user authority bypasses the upstream gate in most tested
cases. CapLease therefore preserves the budget of the authorization
established by the trusted runtime; authenticated contracts and
runtime-produced provenance are required to ensure that this initial
authorization is itself sound.

\subsection{Overall evidence and system implication}
Together, the results answer the three research questions. Uncertain
execution outcomes produce semantic replay across model families, while
token-local and grant-local single use prevent only reuse of the
original artifact. Durable uniqueness and atomic slot transitions are
required to bound consumption across fresh artifacts, and a stable sink
key resolves uncertain completion. These guarantees are not obtained by
rejecting useful work: all legitimate budget configurations complete,
all prepared and committed operations recover, and the detailed benign
set retains 97/97 success. The matched Server Ledger result isolates
durable state as the essential safety primitive. CapLease packages that
primitive as an authenticated and inspectable handoff across separated
proposal, authority, and execution components, without reconstructing
authority from model-generated text.

\section{Conclusion}
\label{sec:conclusion}

We identify \emph{semantic replay}, where fresh execution artifacts
consume one authorization beyond its budget. Across 10,152 agent
trajectories, uncertain outcomes induce equivalent reproposals in
39.8\% of cases, showing that token-local single use is insufficient.
CapLease binds the action, confirmation, and budget to durable
token-independent state. Its Issue--Prepare--Commit protocol prevents
fresh reissuance, excess admission, and duplicate effects while
composing with existing authority mechanisms. A matched Server Ledger
provides the same centralized guarantees, confirming that replay safety
depends on durable authorization state rather than token representation.

\bibliography{references}

\begin{thebibliography}{28}
\providecommand{\natexlab}[1]{#1}

\bibitem[{An et~al.(2025)An, Zhang, Du, Zhou, Li, Lin, and
  Ji}]{an-etal-2025-ipiguard}
An, H.; Zhang, J.; Du, T.; Zhou, C.; Li, Q.; Lin, T.; and Ji, S. 2025.
\newblock {IPIG}uard: A Novel Tool Dependency Graph-Based Defense Against
  Indirect Prompt Injection in {LLM} Agents.
\newblock In Christodoulopoulos, C.; Chakraborty, T.; Rose, C.; and Peng, V.,
  eds., \emph{Proceedings of the 2025 Conference on Empirical Methods in
  Natural Language Processing}, 1023--1039. Suzhou, China: Association for
  Computational Linguistics.
\newblock ISBN 979-8-89176-332-6.

\bibitem[{Bauer, Schneider, and Felten(2002)}]{647253.757044}
Bauer, L.; Schneider, M.~A.; and Felten, E.~W. 2002.
\newblock A General and Flexible Access-Control System for the Web.
\newblock In \emph{Proceedings of the 11th USENIX Security Symposium},
  93–108. USA: USENIX Association.
\newblock ISBN 1931971005.

\bibitem[{Birgisson et~al.(2014)Birgisson, Politz, Erlingsson, Taly, Vrable,
  and Lentczner}]{birgisson2014macaroons}
Birgisson, A.; Politz, J.~G.; Erlingsson, U.; Taly, A.; Vrable, M.; and
  Lentczner, M. 2014.
\newblock Macaroons: Cookies with Contextual Caveats for Decentralized
  Authorization in the Cloud.
\newblock In \emph{NDSS}.

\bibitem[{Birrell(1985)}]{10.1145/214451.214452}
Birrell, A.~D. 1985.
\newblock Secure communication using remote procedure calls.
\newblock \emph{ACM Trans. Comput. Syst.}, 3(1): 1–14.

\bibitem[{Birrell and Nelson(1984)}]{10.1145/2080.357392}
Birrell, A.~D.; and Nelson, B.~J. 1984.
\newblock Implementing remote procedure calls.
\newblock \emph{ACM Trans. Comput. Syst.}, 2(1): 39–59.

\bibitem[{Cao et~al.(2024)Cao, Meng, Stefan, and Fernandes}]{299818}
Cao, L.; Meng, L.; Stefan, D.; and Fernandes, E. 2024.
\newblock Stateful Least Privilege Authorization for the Cloud.
\newblock In \emph{33rd USENIX Security Symposium (USENIX Security 24)},
  3477--3494. Philadelphia, PA: USENIX Association.
\newblock ISBN 978-1-939133-44-1.

\bibitem[{Chennabasappa et~al.(2025)}]{chennabasappa2025llamafirewall}
Chennabasappa, S.; et~al. 2025.
\newblock LlamaFirewall: An Open Source Guardrail System for Building Secure AI
  Agents.
\newblock \emph{arXiv preprint arXiv:2505.03574}.

\bibitem[{Debenedetti et~al.(2025)Debenedetti, Shumailov, Fan, Hayes, Carlini,
  Fabian, Kern, Shi, Terzis, and Tram{\`e}r}]{debenedetti2025camel}
Debenedetti, E.; Shumailov, I.; Fan, T.; Hayes, J.; Carlini, N.; Fabian, D.;
  Kern, C.; Shi, C.; Terzis, A.; and Tram{\`e}r, F. 2025.
\newblock Defeating Prompt Injections by Design.
\newblock \emph{arXiv preprint arXiv:2503.18813}.

\bibitem[{Debenedetti et~al.(2024)Debenedetti, Zhang, Balunovic,
  Beurer-Kellner, Fischer, and Tramer}]{debenedetti2024agentdojo}
Debenedetti, E.; Zhang, J.; Balunovic, M.; Beurer-Kellner, L.; Fischer, M.; and
  Tramer, F. 2024.
\newblock AgentDojo: A Dynamic Environment to Evaluate Prompt Injection Attacks
  and Defenses for LLM Agents.
\newblock In \emph{Advances in Neural Information Processing Systems, Datasets
  and Benchmarks Track}.

\bibitem[{Fan et~al.(2026)Fan, Li, Tian, Wang, Li, and Wang}]{fan2026pact}
Fan, L.; Li, Z.; Tian, Y.; Wang, Y.; Li, R.; and Wang, X. 2026.
\newblock The Granularity Mismatch in Agent Security: Argument-Level Provenance
  Solves Enforcement and Isolates the LLM Reasoning Bottleneck.
\newblock \emph{arXiv preprint arXiv:2605.11039}.

\bibitem[{Gray and Cheriton(1989)}]{10.1145/74850.74870}
Gray, C.; and Cheriton, D. 1989.
\newblock Leases: an efficient fault-tolerant mechanism for distributed file
  cache consistency.
\newblock In \emph{Proceedings of the Twelfth ACM Symposium on Operating
  Systems Principles}, SOSP '89, 202–210. New York, NY, USA: Association for
  Computing Machinery.
\newblock ISBN 0897913388.

\bibitem[{Gray(1981)}]{DBLP:conf/vldb/Gray81}
Gray, J. 1981.
\newblock The Transaction Concept: Virtues and Limitations (Invited Paper).
\newblock In \emph{Very Large Data Bases, 7th International Conference,
  September 9-11, 1981, Cannes, France, Proceedings}, 144--154. IEEE Computer
  Society.

\bibitem[{Herlihy and Wing(1990)}]{10.1145/78969.78972}
Herlihy, M.~P.; and Wing, J.~M. 1990.
\newblock Linearizability: a correctness condition for concurrent objects.
\newblock \emph{ACM Trans. Program. Lang. Syst.}, 12(3): 463–492.

\bibitem[{Howell and Kotz(2000)}]{271165}
Howell, J.; and Kotz, D. 2000.
\newblock {End-to-End} Authorization.
\newblock In \emph{Fourth Symposium on Operating Systems Design and
  Implementation (OSDI 2000)}. San Diego, CA: USENIX Association.

\bibitem[{Jia et~al.(2025)Jia, Wu, Qin, and Squicciarini}]{jia-etal-2025-task}
Jia, F.; Wu, T.; Qin, X.; and Squicciarini, A. 2025.
\newblock The Task Shield: Enforcing Task Alignment to Defend Against Indirect
  Prompt Injection in {LLM} Agents.
\newblock In Che, W.; Nabende, J.; Shutova, E.; and Pilehvar, M.~T., eds.,
  \emph{Proceedings of the 63rd Annual Meeting of the Association for
  Computational Linguistics (Volume 1: Long Papers)}, 29680--29697. Vienna,
  Austria: Association for Computational Linguistics.
\newblock ISBN 979-8-89176-251-0.

\bibitem[{Lin et~al.(2026)Lin, Zhou, Zheng, Liu, Xu, Chen, and
  Chen}]{lin-etal-2026-vigil}
Lin, J.; Zhou, Z.; Zheng, Z.; Liu, S.; Xu, T.; Chen, Y.; and Chen, E. 2026.
\newblock {VIGIL}: Defending {LLM} Agents Against Tool-Stream Injection via
  Verify-Before-Commit.
\newblock In Liakata, M.; Moreira, V.~P.; Zhang, J.; and Jurgens, D., eds.,
  \emph{Proceedings of the 64th Annual Meeting of the {A}ssociation for
  {C}omputational {L}inguistics (Volume 1: Long Papers)}, 9764--9785. San
  Diego, California, United States: Association for Computational Linguistics.
\newblock ISBN 979-8-89176-390-6.

\bibitem[{Qin et~al.(2026)Qin, Zhuang, Zhou, Han, and Zhang}]{qin2026airguard}
Qin, S.; Zhuang, H.; Zhou, Y.; Han, Y.; and Zhang, X. 2026.
\newblock {AIRGuard}: Guarding Agent Actions with Runtime Authority Control.
\newblock \emph{arXiv preprint arXiv:2605.28914}.

\bibitem[{Ruan et~al.(2023)Ruan, Dong, Wang, Pitis, Zhou, Ba, Dubois, Maddison,
  and Hashimoto}]{ruan2023toolemu}
Ruan, Y.; Dong, H.; Wang, A.; Pitis, S.; Zhou, Y.; Ba, J.; Dubois, Y.;
  Maddison, C.~J.; and Hashimoto, T. 2023.
\newblock Identifying the Risks of LM Agents with an LM-Emulated Sandbox.
\newblock \emph{arXiv preprint arXiv:2309.15817}.

\bibitem[{Saltzer and Schroeder(1975)}]{saltzer1975protection}
Saltzer, J.~H.; and Schroeder, M.~D. 1975.
\newblock The Protection of Information in Computer Systems.
\newblock \emph{Proceedings of the IEEE}, 63(9): 1278--1308.

\bibitem[{Santos-Grueiro(2026)}]{santosgrueiro2026cxi}
Santos-Grueiro, I. 2026.
\newblock Context-to-Execution Integrity for {LLM} Agents.
\newblock \emph{arXiv preprint arXiv:2607.06000}.

\bibitem[{Schick et~al.(2023)Schick, Dwivedi-Yu, Dessi, Raileanu, Lomeli,
  Zettlemoyer, Cancedda, and Scialom}]{schick2023toolformer}
Schick, T.; Dwivedi-Yu, J.; Dessi, R.; Raileanu, R.; Lomeli, M.; Zettlemoyer,
  L.; Cancedda, N.; and Scialom, T. 2023.
\newblock Toolformer: Language Models Can Teach Themselves to Use Tools.
\newblock In \emph{Advances in Neural Information Processing Systems}.

\bibitem[{Shi et~al.(2025{\natexlab{a}})Shi, He, Wang, Wu, Li, Guo, and
  Song}]{shi2025progent}
Shi, T.; He, J.; Wang, Z.; Wu, L.; Li, H.; Guo, W.; and Song, D.
  2025{\natexlab{a}}.
\newblock {Progent}: Programmable Privilege Control for {LLM} Agents.
\newblock \emph{arXiv preprint arXiv:2504.11703}.

\bibitem[{Shi et~al.(2025{\natexlab{b}})Shi, Zhu, Wang, Jia, Cai, Liang, Wang,
  Alzahrani, Lu, Kawaguchi, Alomair, Zhao, Wang, Gong, Guo, and
  Song}]{shi2025promptarmor}
Shi, T.; Zhu, K.; Wang, Z.; Jia, Y.; Cai, W.; Liang, W.; Wang, H.; Alzahrani,
  H.; Lu, J.; Kawaguchi, K.; Alomair, B.; Zhao, X.; Wang, W.~Y.; Gong, N.; Guo,
  W.; and Song, D. 2025{\natexlab{b}}.
\newblock PromptArmor: Simple yet Effective Prompt Injection Defenses.
\newblock \emph{arXiv preprint arXiv:2507.15219}.

\bibitem[{Yao et~al.(2023)Yao, Zhao, Yu, Du, Shafran, Narasimhan, and
  Cao}]{yao2023react}
Yao, S.; Zhao, J.; Yu, D.; Du, N.; Shafran, I.; Narasimhan, K.; and Cao, Y.
  2023.
\newblock ReAct: Synergizing Reasoning and Acting in Language Models.
\newblock In \emph{International Conference on Learning Representations}.

\bibitem[{Yu et~al.(2026)Yu, Geng, Zeng, and Knottenbelt}]{yu2026sudp}
Yu, X.; Geng, H.; Zeng, X.; and Knottenbelt, W. 2026.
\newblock {SUDP}: Secret-Use Delegation Protocol for Agentic Systems.
\newblock \emph{arXiv preprint arXiv:2604.24920}.

\bibitem[{Zhan et~al.(2024)Zhan, Liang, Ying, and
  Kang}]{zhan-etal-2024-injecagent}
Zhan, Q.; Liang, Z.; Ying, Z.; and Kang, D. 2024.
\newblock {I}njec{A}gent: Benchmarking Indirect Prompt Injections in
  Tool-Integrated Large Language Model Agents.
\newblock In Ku, L.-W.; Martins, A.; and Srikumar, V., eds., \emph{Findings of
  the Association for Computational Linguistics: ACL 2024}, 10471--10506.
  Bangkok, Thailand: Association for Computational Linguistics.

\bibitem[{Zhang et~al.(2025)Zhang, Huang, Mei, Yao, Wang, Zhan, Wang, and
  Zhang}]{zhang2025iclr-agent}
Zhang, H.; Huang, J.; Mei, K.; Yao, Y.; Wang, Z.; Zhan, C.; Wang, H.; and
  Zhang, Y. 2025.
\newblock {Agent Security Bench (ASB): Formalizing and Benchmarking Attacks and
  Defenses in LLM-Based Agents}.
\newblock In \emph{International Conference on Learning Representations}.

\bibitem[{Zhu et~al.(2025)Zhu, Yang, Wang, Guo, and Wang}]{zhu2025melon}
Zhu, K.; Yang, X.; Wang, J.; Guo, W.; and Wang, W.~Y. 2025.
\newblock MELON: Provable Defense Against Indirect Prompt Injection Attacks in
  AI Agents.
\newblock \emph{arXiv preprint arXiv:2502.05174}.

\end{thebibliography}

\end{document}